\documentclass{article} % For LaTeX2e
\usepackage[utf8]{inputenc}
\usepackage{iclr2020_conference,times}

\usepackage{amsmath,amsfonts,bm}

\def\eqref#1{equation~\ref{#1}}
\def\1{\bm{1}}

\DeclareMathAlphabet{\mathsfit}{\encodingdefault}{\sfdefault}{m}{sl}
\SetMathAlphabet{\mathsfit}{bold}{\encodingdefault}{\sfdefault}{bx}{n}

\usepackage{hyperref}
\usepackage{url}
\usepackage{graphicx}
\usepackage{dsfont}
\usepackage{wrapfig}

\usepackage{pythonhighlight}

\usepackage{thmtools,thm-restate}
\usepackage{subcaption} 
\usepackage{cleveref}

\usepackage{amsmath,amsthm,amssymb}
\usepackage{thm-restate}

\newtheorem{theorem}{Theorem}

\newtheorem{lemma}{Lemma}

\definecolor{myBlue}{RGB}{55, 126, 184}
\definecolor{myGreen}{RGB}{77, 175, 74}

\iclrfinalcopy

\author{
Alexander Mathiasen\thanks{Aarhus University, \{alexander.mathiasen, fhvilshoj\}@gmail.com} 
\And
Frederik Hvilsh\o j\footnotemark[1]
}

\title{One Reflection Suffice}

\begin{document}

\maketitle

\begin{abstract}
%Orthogonal weight matrices are used in many areas of deep learning.
%Much previous work attempt to alleviate the additional compute it requires to constrain weight matrices to be orthogonal. 
%One approach utilize \emph{many} sequential reflections, which is slow on parallel hardware like GPUs.
%We speed up this approach by circumventing the \emph{many} sequential reflections. In particular, we prove that \emph{one} reflection is sufficient, if the reflection is computed by an auxiliary neural network.
Orthogonal weight matrices are used in many areas of deep learning. 
Much previous work attempt to alleviate the additional computational resources it requires to constrain weight matrices to be orthogonal. 
One popular approach utilizes \emph{many} Householder reflections. 
The only practical drawback is that many reflections cause low GPU utilization. 
We mitigate this final drawback by proving that \emph{one} reflection is sufficient, if the reflection is computed by an auxiliary neural network.
%We develop further theory for such ``auxiliary reflections'' to allow their use in Normalizing Flows. %We develop further theory regarding the Jacobian and inverse of such auxiliary reflections to allow their use in Normalizing Flows. 
%We ur%Such ``auxiliary rhteflections''

%It is known that every orthogonal matrix can be represented by a sequence of such reflections. 
%This yields the desired orthogonal transform in the best possible time complexity on a sequential machine. 
% drawback of reflection based approaches is low GPU utilization. 
%This yields the both of both worlds, in terms of parallelization and time complexity. 
\end{abstract}

\section{Introduction}
Orthogonal matrices have shown several benefits in deep learning, with successful applications in Recurrent Neural Networks, Convolutional Neural Networks and Normalizing Flows. 
%Recurrent Neural Networks can circumvent exploding and vanishing gradients with an orthogonal transition matrix \citep{unitaryrnn}. 
%Convolutional Neural Networks optimize faster by using orthogonal kernels \citep{orthcnn}. 
%Normalizing Flows admit an easy inverse for orthogonal matrices, $W^T=W^{-1}$, and circumvents the need to compute expensive Jacobian determinants \citep{emerging}. 
%Deep learning has attained state of the art in diverse fields as A and B. 
%Deep learning continues to improve with more data and compute \citep{gpt3}. 
%Algorithmic advances that reduce compute are thus highly desirable. 
%One line of research demonstrate that neural networks converge faster when their weight matrices are orthogonal. %training convergence when neural networks have orthogonal weight matrices. %A recent line of work demonstrates several benefit to orthogonal parameterization in deep learning with widespread application, including Recurrent Neural Networks, Convolutions Neural Networks, Transformers and Normalizing Flows.
%This emphasizes the need for architectures as any reduction in alputs a huge need for One of the main drawbacks of deep learning, is the need for huge amounts of computation, in particular, training recent models have increased. 
%One approach to tackle this issueA recent 
%There are several ways to utilize orthogonal matrices in neural networks. 
One popular approach can represent any $d\times d$ orthogonal matrix using $d$ Householder reflections \citep{hhrnn}. 
The only practical drawback is low GPU utilization, which happens because the $d$ reflections needs to be evaluated sequentially \citep{fasth}. %needs to be evaluated sequentialy, which prevent utilizing all parallel cores on GPUs. %parallel hardware like GPUs \cite{}. %, typically limiting the GPU to utilize only $1-5\%$ of its cores. 
Previous work often increases GPU utilization by using $k\ll d$ reflections \citep{hhnf,hhrnn,svdnn,sylvesternf}. 
Using fewer reflections limits the orthogonal transformations the reflections can represent, yielding a trade-off between representational power and computation time. 
%We present an alternative that can represent any orthogonal transformation using a single reflection, attaining both advantages by circumventing the trade-off. 
This raises an intriguing question: can we circumvent the trade-off and attain full representational power without sacrificing computation time? % is it really necessary to trade computation time for representational power, could we get have both? %off representational power for and computation time, could we not get both? %This raises the question whether the trade-off is really necessary, i.e., is it possible to have full representatial power without sacrificing computation time?% between representational power and computational resources is really necessary. % enjoying the advantages of both extremes from the trade-off.

We answer this question with a surprising ``yes.''
%To this end, we introduce a novel type of reflection. 
The key idea is to use an auxiliary neural network to compute a different reflection for each input. 
In theory, we prove that \emph{one} such ``auxiliary reflection'' can represent any number of normal reflections. 
In practice, we demonstrate that one auxiliary reflection attains similar validation error to models with $d$ normal reflections, when training Fully Connected Neural Networks (\Cref{fig:time} left), 
Recurrent Neural Networks (\Cref{fig:time} center) and
convolutions in Normalizing Flows (\Cref{fig:time} right). 
Notably, auxiliary reflections train between $2$ and $6$ times faster for Fully Connected Neural Networks with orthogonal weight matrices (see \Cref{sec:exp}). %with orthogonal weight matrices. %, the auxiliary reflection reduce training time $2\times$ to $6\times $. 

%\Cref{sec:hypo} introduces the auxiliary reflector, and proves that it is at least as ``expressive'' as normal orthogonal layers. 
%In \Cref{sec:exp}, we compare the auxiliary reflector against the usual reflection based approaches when training Recurrent Neural Networks and Normalizing Flows. 
%Normalizing Flows require computing inverse and the jacobian determinant of the auxiliary reflector, we introduce provably efficient algorithms for both in \Cref{sec:hypo_inv} and \Cref{sec:hypo_jac}, respectively. 

\begin{figure}[h]
    \centering
    \includegraphics[width=1\textwidth]{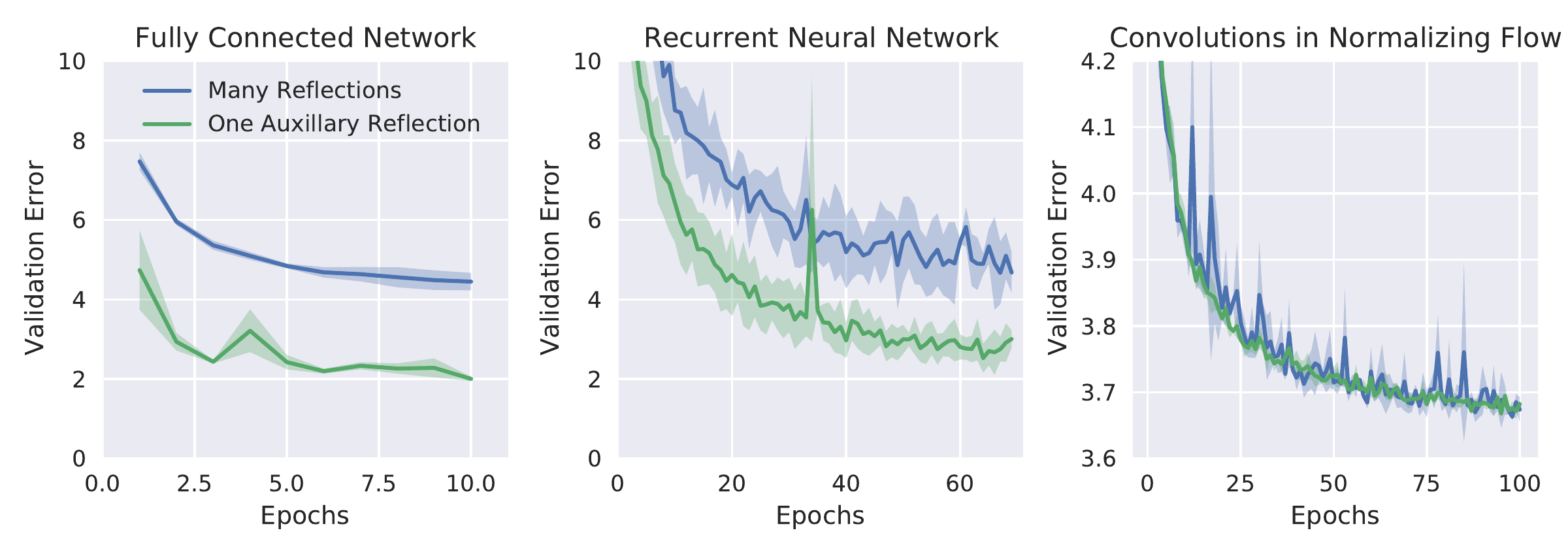}
    \caption{Models with one auxiliary reflection attains similar validation error to models with many reflections. 
    %for Fully Connected Network (left), Normalizing Flow (center) and a Recurrent Neural Network (right). 
    Lower error means better performance. 
    See \Cref{sec:exp} for details. }
    \label{fig:time}
\end{figure}

\subsection{Our Results}\label{sec:hypo}
The Householder reflection of $x\in\mathbb{R}^d$ around $v\in \mathbb{R}^d$ can be represented by a matrix $H(v)\in \mathbb{R}^{d\times d}$. 
$$H(v)x = \left(I-2\frac{vv^T}{||v||^2}\right)x.$$
An auxiliary reflection uses a Householder matrix $H(v)$ with $v=n(x)$ for a neural network $n$. 
$$f(x)=H(n(x))x=\left(I-2\frac{n(x)n(x)^T}{||n(x)||^2}\right)x. $$
One auxiliary reflection can represent any composition of Householder reflections. 
We prove this claim even when we restrict the neural network $n(x)$ to have a single linear layer $n(x)=Wx$ for $W\in\mathbb{R}^{d\times d}$ such that $f(x)=H(Wx)x$. %We defer the proof of \Cref{thm:orth} to \Cref{app:lemma:ref}.% (see \Cref{app:lemma:ref} for our proof). % (see \Cref{app:lemma:ref} for a proof). %This is true even when we restrict $n(x)$ to a single linear layer $n(x)=Wx$.
\begin{restatable}{theorem}{thmorth}
\label{thm:orth}
For any $k$ Householder reflections $U=H(v_1)\cdots H(v_k)$ there exists a neural network $n(x)=Wx$ with $W\in \mathbb{R}^{d\times d}$ such that 
$f(x)=H(Wx)x=Ux$ for all $x\in \mathbb{R}^d\backslash\{ 0\}$. 
\end{restatable}
%\begin{proof}\Cref{app:lemma:ref}.  \end{proof}
%\begin{proof}The result is an immediate consequence of a well-known lemma concerning Householder reflections. 
%For ease of exposition we refer the details to \Cref{app:lemma:ref}.
%Let $X=(x_1,...,x_k)\in \mathbb{R}^{d\times k}$ be the matrix with $x_i$ as columns. For simplicity, we first show the case where $X$ has a left inverse $X^{-1}\in\mathbb{R}^{k\times d}$ so $X^{-1}X=I_k$. By definition: 
%$$\text{aux\_ref}(W, x_i)=(I-2Wx_ix_i^TW^T/||Wx_i||^2)x_i.$$ 
%Let $W=VX^{-1}\in \mathbb{R}^{d\times d}$ for some $V\in \mathbb{R}^{d\times k}$ with columns $v_i\in\mathbb{R}^d$, then 
%\begin{align}
%\text{aux\_rfl}(x_i)&=(I-2VX^{-1}x_ix_i^TX^{T^{-1}}V^T/||VX^{-1}x_i||^2)x_i\\
%&=(I-2Ve_ie_i^TV^T/||Ve_i||^2)x_i\\
%&=(I-2v_iv_i^T/||v_i||^2)x_i =H(v_i)x_i
%\end{align}
%We need $v_i$ so $H(v_i)x_i=Ux_i$. 
%By \Cref{lemma:ref}, this is just $v_i=x_i-Ux_i=(I-U)x_i$\frederik{I would change the order here: $v_i = x_i - Ux_i$ then it fits  lemma better, when $x$ is $x$ and $Ux$ is $y$}. Putting this together, we choose $W=(I-U)XX^{-1}$.  \todo{revise this}
%\end{proof}
%auxiliary reflections thus allow us to represent any modeBy \Cref{thm:orth} it is possible to circumvent the trade-off. ... 

%While previois work did somethign with $k<d$ reflections, thm1 is the first that statement that shows one can that one reflection has sufficient representational power.

Previous work \citep{hhrnn,svdnn} often employ $k\ll d$ reflections and compute $Ux$ as $k$ sequential Householder reflections $H(v_1)\cdots H(v_k)\cdot x$ with weights $V=(v_1\; \cdots \;v_k)$. 
It is the evaluation of these sequential Householder reflection that cause low GPU utilization \citep{fasth}, so lower values of $k$ increase GPU utilization but decrease representational power. %the sequential evaluation The decrease in GPU utilization is caus%The sequential evaluation of Householder reflections are ill-suited for parallel hardware like GPUs, which cause low GPU utilization \citep{fasth}. 
\Cref{thm:orth} states that it is sufficient to evaluate a single auxiliary reflection $H(Wx)x$ instead of $k$ reflections $H(v_1)\cdots H(v_k)\cdot x$, thereby gaining high GPU utilization while retaining the full representational power of any number of reflections.
%See \Cref{app:lemma:ref} for our proof of \Cref{thm:orth}. 

%The sequential evaluation of Householder reflections $H(v_1)\cdots H(v_k)x$ is illsuited for parallel becomes the computational reduction in sequential The Previous work \citep{fasth} mention sequential stuff is slow on GPUs, and conclude that we circumvent this . This allows us to circumvent the previous trade-off between sequential 
%Auxiliary reflections thus attain one sequential operation compare to $k$ The auxiliary reflection thus retain full circumvents the sequential evalaution of $H(v_1)\cdots H(v_d) \cdot x$ responsible for the decrease in GPU utilization \citep{fasth}, while retaining full representational power. %computation% circumventing any sequential computations. %, see \Cref{app:refl} for a small PyTorch example. 

%\Cref{thm:orth} tells us that, in \emph{theory}, one auxiliary reflection can represent any number of Householder reflections. 
%\Cref{sec:exp} investigates what happens with a neural network, 
In \emph{practice}, we demonstrate that $d$ reflections can be substituted with a single auxiliary reflection without decreasing validation error, when training Fully Connected Neural Networks (\Cref{exp:fcn}), Recurrent Neural Networks (\Cref{exp:rnn}) and Normalizing Flows (\Cref{exp:nf}). 
%In particular, we demonstrate that one auxiliary reflection attain similar performance to models with many reflections for several different types of neural networks, including Fully Connected Neural Network (\Cref{exp:fcn}), the Recurrent Neural Network from \citep{hhrnn} (\Cref{exp:rnn}). 
%and the Normalizing Flow from \citep{emerging} (\Cref{exp:nf}). 
%These experiments are summarized in \Cref{fig:time}. 
%Auxiliary reflections thus allow us to circumvent slow sequential computations caused by $H(v_1)\cdots H(v_d)\cdot x$ while retaining full representational power. 
%However, it might be that experimental performance %while retaining full representational power, eliminating the need for a trade-off between the two. %For concreteness, we added a small PyTorch example in \Cref{app:refl} that compares normal reflections against sequential reflection. 
%. In this way, auxiliary reflections circumvent the tradethe trade-off, and reach full representational power without sacrificing computational power.  
While the use of auxiliary reflections is straightforward for Fully Connected Neural Networks and Recurrent Neural Networks, we needed additional ideas to support auxiliary reflections in Normalizing Flows. 
In particular, we developed further theory concerning the inverse and Jacobian of $f(x)=H(Wx)x$. %xauxiliary reflections. %\Cref{sec:nf} proves \Cref{thm:tinv} concerning the invertibility of auxiliary reflections, and \Cref{lemma:jac} concerning the Jacobian of auxiliary reflections.% further theory concerning the invertibility and Jacobian determinant of auxiliary reflections (\Cref{subsec:jac}). 
 Note that $f$ is invertible if there exists a unique $x$ given $y=H(Wx)x$ and $W$.% an auxiliary reflection is invertible if there exists a unique $x$ given $y=H(Wx)x$ and $W$. 
%We prove conditions on $W$ that imply invertible of auxiliary reflections. %In particular, we prove conditions for which auxiliary reflections are invertible, i.e., given $y=H(Wx)x$ and $W$ when does there exist a unique $x$?

\begin{restatable}{theorem}{tinv}\label{thm:tinv}
Let $f(x)=H(Wx)x$ with $f(0):=0$, then $f$ is invertible on $\mathbb{R}^d$ with $d\ge 2$ if $W=W^T$ and has eigenvalues which satisfy $3/2\cdot \lambda_{\min}(W) > \lambda_{\max}(W)$. 
\end{restatable}
Finally, we present a matrix formula for the Jacobian of the auxiliary reflection $f(x)=H(Wx)x$.
This matrix formula is used in our proof of \Cref{thm:tinv}, but it also allows us simplify the Jacobian determinant (\Cref{lemma:jacdet}) which is needed when training Normalizing Flows. %needed when training Normalizing Flows. %which allow us to simplify Jacobian determinant computations (see \Cref{lemma:jacdet}). %This allows us to simplify the Jacobian determinant computatons In particular, this allows us to simplify Jacobian determinant computations, see \Cref{} for details. %These results, together with the practical inverse computations, are presented in \Cref{sec:nf}.

\begin{restatable}{theorem}{lemmajac}\label{lemma:jac}
The Jacobian of $f(x)=H(Wx)x$ is: 
$$J=H(Wx)A-2\frac{Wxx^TW}{||Wx||^2}\quad\text{where}\quad A=I-2\frac{x^TW^Tx}{||Wx||^2}W.$$
\end{restatable}
We prove \Cref{thm:orth} in \Cref{app:lemma:ref} while \Cref{lemma:jac,thm:tinv} are proved in \Cref{sec:nf}. 
%\Cref{thm:inv} all
%\paragraph{Orthogonal Matrices. }An auxiliary reflection can represent any orthogonal matrix, because the orthogonal matrix can be written as a product of Householder reflections~\citep{hhrnn}. 
%auxiliary reflections can further represent non-linear norm-preserving transformations which orthogonal matrices cannot. 
%To see this, note that auxiliary transformations are always norm-preserving $||\text{aux\_ref}(x)||=||H(n(x))x||$ because $H(n(x))$ is orthogonal, however, it is possible that they are non-linear $\text{aux\_ref}(x+y)\neq \text{aux\_ref}(x) +\text{aux\_ref}(y)$. 
%In all cases, we find that one auxiliary reflection attain similar performance to multiple reflections. 
%However, it is not straight-forward to use auxiliary reflections in Normalizing Flows. 
%This require the development of theory which we present in \Cref{sec:nf}. 
\section{Normalizing Flows}\label{sec:nf}
\subsection{Background}
Let $z\sim N(0,1)^d$ and $f$ be an invertible neural network. Then $f^{-1}(z)\sim P_{model}$ defines a model distribution for which we can compute likelihood of $x\sim P_{data}$ \citep{nice}. 
\begin{equation}
\label{equ:nll}
     \log p_{model}(x)=\log p_z(f(x)) + \log \left| \det \left( \frac{\partial f(x)}{\partial x }\right)\right|
\end{equation}
This allows us to train invertible neural network as generative models by maximum likelihood. Previous work demonstrate how to construct invertible neural networks and efficiently compute the log jacobian determinant \citep{realnvp,glow,flowpp}. 
%To this end, several authors \cite{emerging,sylvesternf} have found orthognoal matrices to be useful, since they are always invertbile $U=U^T$ and admit unit Jacobian determinant $\log \left| \det U \right|=0$. 

\subsection{Invertibility and Jacobian Determinant (Proof Sketch)}\label{subsec:jac}
To use auxiliary reflections in Normalizing Flows we need invertibility. 
That is, for every $y\in \mathbb{R}^d$ there must exist a unique $x\in\mathbb{R}^d$ so $f(x)=H(Wx)x=y$.\footnote{Note that we do not know $H(Wx)$ so we cannot trivially compute $x=H(Wx)^{-1}y=H(Wx)y$. } 
We find that $f$ is invertible if its Jacobian determinant is non-zero for all $x$ in $S^{d-1}=\{x \in \mathbb{R}^d \mid \|x\|=1\}$. 
\begin{restatable}{theorem}{thminv}
\label{thm:inv} 
Let $f(x)=H(Wx)x$ with $f(0):=0$, then $f$ is invertible on $\mathbb{R}^d$ with $d\ge 2$ if the Jacobian determinant of $f$ is non-zero for all $x\in S^{d-1}$ and $W$ is invertible. 
\end{restatable}
The Jacobian determinant of $H(Wx)x$ takes the following form. 
\begin{restatable}{lemma}{lemmajacdet}
\label{lemma:jacdet}
The Jacobian determinant of $f(x)=H(Wx)x$ is: 
\begin{align*}
-\det(A)\left(1+ 2\frac{v^TA^{-1} u}{||u||^2}\right)\text{ where }v^T=x^TW , u=Wx \text{ and } A=I-2\frac{x^TW^Tx}{||Wx||^2}W.
\end{align*}
\end{restatable}
It is then sufficient that $\det(A)\neq 0$ and $1+2v^TA^{-1}u/||u||^2 \neq 0$. 
We prove that this happens if $W=W^T$ with eigenvalues $3/2\cdot \lambda_{\min}(W) > \lambda_{\max}(W)$. 
This can be achieved with $W=I+VV^T$ if we guarantee $\sigma_{\max}(VV^T)<1/2$ by spectral normalization \citep{snorm}. 
Combining these results yields \Cref{thm:tinv}. 
\tinv*
\paragraph{Computing the Inverse. }In practice, we use Newtons method to compute $x$ so $H(Wx)x=y$. 
\Cref{fig:inv} show reconstructions $n^{-1}(n(x))=x$ for an invertible neural network $n$ with auxiliary reflections using Newtons method, 
see \Cref{app:inv} for details. 
\begin{figure}[h!]
    \centering
    \includegraphics[width=\textwidth]{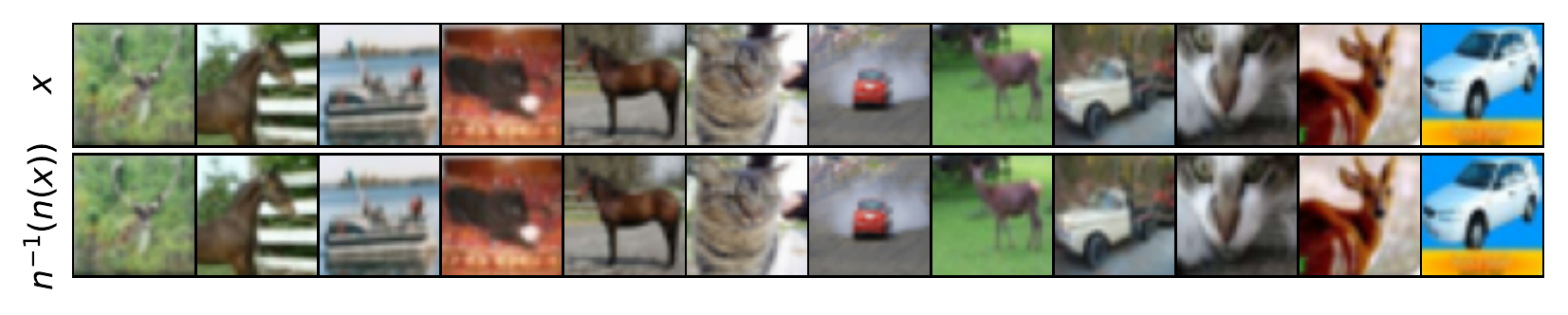}
    \caption{CIFAR10 \citep{cifar} images $x$ and reconstructions $n^{-1}(n(x))$ for an invertible neural network $n$ called Glow \citep{glow}. 
    The network uses auxiliary reflections and we compute their inverse using Newtons method, see \Cref{app:inv} for details.}
    \label{fig:inv}
\end{figure}

\subsection{Proofs}\label{sec:proof}
The goal of this section is to prove that $f(x)=H(Wx)x$ is invertible. 
Our proof strategy has two parts. 
\Cref{sss:suffice} first shows $f$ is invertible if it has non-zero Jacobian determinant. 
\Cref{sss:nnz} then present an expression for the Jacobian determinant, \Cref{lemma:jacdet}, and prove the expression is non-zero if $W=W^T$ and $3/2\cdot \lambda_{\min}(W)> \lambda_{\min}(W)$. %It immediately that $f$ is invertible under the constraints. 
%Combining both results guarantee that $f$ is invertible. 

\subsubsection{Non-Zero Jacobian Determinant Implies Invertibility }\label{sss:suffice}
In this section, we prove that $f(x)=H(Wx)x$ is invertible on $\mathbb{R}^d$ if $f$ has non-zero Jacobian determinant.
To simplify matters, we first prove that invertibility on $S^{d-1}$ implies invertibility on $\mathbb{R}^d$. Informally, invertibility on  $S^{d-1}$ is sufficient because $H(Wx)$ is scale invariant, i.e., $H(c\cdot Wx)=H(Wx)$ for all $c\neq 0$. This is formalized by \Cref{lemma:sd}. 
\begin{lemma}\label{lemma:sd}
If $f(x)=H(Wx)x$ is invertible on $S^{d-1}$ it is also invertible on $\mathbb{R}^d\backslash \{0\}$.  
\end{lemma}
\begin{proof}
Assume that $f(x)$ is invertible on $S^{d-1}$. Pick any $y'\in \mathbb{R}^d$ such that  $||y'||=c$ for any $c>0$. 
Our goal is to compute $x'$ such that $H(Wx')x'=y'$. 
By normalizing, we see $y'/\|y'\|\in S^{d-1}$.
We can then use the inverse $f^{-1}$ on $y'/\|y'\|$ to find $x$ such that $H(Wx)x=y'/\|y\|$. 
The result is then $x'=x\|y\|$ since $H(Wx')x'=H(Wx)x||y||=y$ due to scale invariance of $H(Wx)$. 
\end{proof}

%\clearpage 
The main theorem we use to prove invertibiliy on $S^{d-1}$ is a variant of \emph{Hadamards global function inverse theorem} from \citep{hadamardthm}. 
On a high-level, Hadamard's theorem says that a function is invertible if it has non-zero Jacobian determinant and satisfies a few additional conditions. 
It turns out that these additional conditions are meet by any continuously differentiable function $f(x)$ when (in the notation of \Cref{thm:hadamard}) $M_1=M_2=S^{d-1}$. 
\begin{theorem}\citep[6.2.8]{hadamardthm}\label{thm:hadamard}
Let $M_1$ and $M_2$ be smooth, connected $N$-dimensional manifolds and let $f:M_1\rightarrow M_2$ be continuously differentiable. If (1) $f$ is proper, (2) the Jacobian of $f$ is non-zero, and (3) $M_2$ is simple connected, then $f$ is invertible. % a homeomorphism (which implies invertibility). 
\end{theorem}
For $M_1=M_2=S^{d-1}$ the additional conditions are met if $f$ is continuously differentiable. 
\begin{restatable}{corollary}{cor}\label{cor}
Let $f:S^{d-1}\rightarrow S^{d-1}$ with $d\ge 2$ be continuously differentiable with non-zero Jacobian determinant, then $f$ is invertible. 
\end{restatable}
\begin{proof}
Note that $S^{d-1}$ is smooth and simply connected if $d\ge 2$ \citep{sdsmooth}. 
Continuously functions on $S^{d-1}$ are proper. 
We conclude $f$ is invertible on $S^{d-1}$ by \Cref{thm:hadamard}. %\alex{refer to standard textbook for these results}. 
\end{proof}
We now show that $f(x)=H(Wx)x$ is continuously differentiable on $S^{d-1}$. 
\begin{lemma} The function $f(x)=H(Wx)x$ is continuously differentiable on $S^{d-1}$ if $W$ is invertible. \label{lemma:c1}
\end{lemma}
\begin{proof}
Compositions of continuously differentiable functions are continuously differentiable by the chain rule. All the functions used to construct $H(Wx)x$ are continuously differentiable, except the division. 
However, the only case where division is not continously differentiable is when $||Wx||=0$. 
Since $W$ is invertible, $||Wx||=0$ iff $x=0$. But $0\notin S^{d-1}$ and we conclude $f$ is continuously differentiable on $S^{d-1}$. 
\end{proof}

\thminv*
%\begin{lemma}
%The function $f(x)$ is invertible on $\mathbb{R}^d$ if it has non-zero Jacobian determinant on $S^{d-1}$. 
%\end{lemma}
\begin{proof}
By \Cref{lemma:c1}, we see $f$ is continuously differentiable since $W$ is invertible, which by \Cref{cor} means $f$ is invertible on $S^{d-1}$ if $f$ has non-zero Jacobian determinant on $S^{d-1}$. 
By \Cref{lemma:sd}, we get that $f$ is invertible on $\mathbb{R}^d$ if it has non-zero Jacobian on $S^{d-1}$. 
\end{proof}

%\thminv*

%\begin{lemma}
%If $f(x)=H(Wx)x$ with $f(0):=0$ is invertible on $S^{d-1}$ it is also invertible on $\mathbb{R}^d$.  
%\end{lemma}
%\begin{proof}
%\Cref{lemma:ainv} tells us $f$ is invertible on $S^d$ if $f$ is continuous and has non-zero Jacobian determinant. The non-zero Jacobian part is proven in \Cref{lemma:nonzero}. 
%It is left to prove (1) $f$ is $C^1$ and (2) invertibility on $S^d$ implies invertibility on $R^d$. 
%
%
%(2) Let $f$ be invertible on $S^d$ . 
%Let $||y'||>c$ for any $c>0$, we want to compute $x'$ such that $H(Wx')x'=y'$. 
%By normalizing we see $y'/\|y'\|^2\in S^d$.
%We can then use $f^{-1}$ on $y'/\|y'\|^2$ to find $x$ such that $H(Wx)x=y'/||y||^2$. 
%The result is then $x'=x||y||$ since $H(Wx')x'=H(Wx)x||y||=y$ due to scale invariance of $H(Wx)$. 
%\end{proof}

\subsubsection{Enforcing Non-Zero Jacobian Determinant }\label{sss:nnz}
The goal of this section is to present conditions on $W$ that ensures the Jacobian determinant of $f(x)$ is non-zero for all $x\in S^{d-1}$. 
We first present a matrix formula for the Jacobian of $f$ in \Cref{lemma:jac}. 
By using the \emph{matrix determinant lemma}, we get a formula for the Jacobian determinant in \Cref{lemma:jacdet}. 
By investigating when this expression can be zero, we finally arive at \Cref{lemma:eig} which states that the Jacobian determinant is non-zero (and $f$ thus invertible) if $W=W^T$ and $3/2\cdot \lambda_{\min}>\lambda_{\max}$.  

%\begin{restatable}{lemma}{lemmajac}\label{lemma:jac}
%The Jacobian of $H(Wx)x$ is: 
%$$J=H(Wx)A-2\frac{Wxx^TW}{||Wx||^2}\quad\text{where}\quad A=I-2\frac{x^TW^Tx}{||Wx||^2}W$$
%\end{restatable}
\lemmajac*

See \Cref{app:autograd} for PyTorch implementation of $J$ and a test case against PyTorch autograd. 
%This matrix form is useful because we can write the Jacobian as a rank one update $M+uv^T$ where $M=H(Wx)A$ and $u=-\sqrt{2}Wx/||Wx||$ and $v^T=\sqrt{2} x^TW/||Wx||$. 
\begin{proof}
The $(i,j)$'th entry of the Jacobian determinant is, by definition, 
$$\frac{\partial (x-2\cdot \frac{Wxx^TW^Tx}{||Wx||^2})_i}{\partial x_j}=\mathds{1}_{i =j}- 2 \cdot \frac{\partial (Wx)_i\cdot \frac{x^TW^Tx} {||Wx||^2}}{\partial x_j}.
$$
Then, by the product rule, we get
\begin{align*}
\frac{\partial (Wx)_i\cdot \frac{x^TW^Tx} {||Wx||^2}}{\partial x_j}&=
\frac{\partial (Wx)_i}{\partial x_j}\cdot \frac{x^TW^Tx}{||Wx||^2} + (Wx)_i \cdot \frac{\partial \frac{x^TW^Tx} {||Wx||^2}}{\partial x_j} \\
&=W_{ij} \cdot \frac{x^TW^Tx}{||Wx||^2} \hspace{6.75mm}+ 
(Wx)_i \cdot \frac{\partial x^TW^Tx \cdot \frac{1}{||Wx||^2}}{\partial x_j} .
\end{align*}
The remaining derivative can be found using the product rule. 
\begin{align*}
\frac{\partial x^TW^Tx \cdot \frac{1}{||Wx||^2}}{\partial x_j} &= 
\frac{\partial x^TW^Tx }{\partial x_j} \cdot \frac{1}{||Wx||^2} +
 x^TW^Tx \cdot \frac{\partial \frac{1}{||Wx||^2}}{\partial x_j}.
\end{align*}
First, \citep{cookbook} equation (81) gives $\frac{\partial x^TW^Tx}{\partial x_j}=((W^T+W)x)_j$. Second $||Wx||^{-2}$ can be found using the chain rule: 
\begin{align*}
    \frac{\partial (||Wx||^2)^{-1}}{\partial x_j}&=
    \frac{\partial (||Wx||^2)^{-1}}{\partial ||Wx||^2}
    \frac{\partial ||Wx||^2}{\partial x_j}
    \\&= 
    -\frac{1}{||Wx||^4}
    \left( \frac{\partial x^TW^TWx}{\partial x}\right)_j 
    \\&= -\frac{1}{||Wx||^4}
    ((W^TW+(W^TW)^T)x)_j &\text{\citep[equ. 81]{cookbook} }
    \\&= -\frac{1}{||Wx||^4}
    2(W^TWx)_j .
\end{align*}
Combining everything we get 
$$J_{ij} = \mathds{1}_{i=j}-2 \left[ \frac{x^TW^Tx}{||Wx||^2}\cdot W_{ij}  + (Wx)_i 
\left( 
\frac{1}{||Wx||^2}\cdot ((W^T+W)x)_j - \frac{2x^TW^Tx}{||Wx||^4} \cdot (W^TWx)_j 
\right) \right]. $$
In matrix notation, this translates into the following, if we let $A=I-2\cdot \frac{x^TW^Tx}{||Wx||^2}W$. 
\begin{align*}
J&=I- 2\left[ 
\frac{x^TW^Tx}{||Wx||^2} \cdot W + Wx 
\left(\frac{1}{||Wx||^2}\cdot x^T(W+W^T) - \frac{2x^TW^Tx}{||Wx||^4}\cdot x^TW^TW\right)
\right]\\
&=I- 2 
\cdot \frac{x^TW^Tx}{||Wx||^2} \cdot W - 
2\cdot \frac{Wxx^TW}{||Wx||^2} - 
2\cdot \frac{Wxx^TW^T}{||Wx||^2} \left(I-2\cdot \frac{x^TW^Tx}{||Wx||^2}W\right)\\
&=A - 
2\cdot \frac{Wxx^TW}{||Wx||^2} - 
2\cdot \frac{Wxx^TW^T}{||Wx||^2} A\\
&=\left(I- 2\cdot \frac{Wxx^TW^T}{||Wx||^2} \right)A-
2\cdot \frac{Wxx^TW}{||Wx||^2} =
H(Wx)A - 2\cdot \frac{Wxx^TW}{||Wx||^2}. 
\end{align*}
This concludes the proof. 
\end{proof}
\Cref{lemma:jac} allows us to write $J$ as a rank one update $M+ab^T$ for $a,b\in\mathbb{R}^d$, which can be used to simplify $\det(J)$ as stated in the following lemma. 
%The \emph{matrix determinant lemma} states that $\det(M+ab^T)=\det(M)(1+b^TM^{-1}a)$. 
%This allows us to attain the following lemma. 

\lemmajacdet*

\begin{proof}
The \emph{matrix determinant lemma} allows us to write $\det(M+ab^T)=\det(M)(1+b^TM^{-1}a)$. 
Let $M=H(Wx)A$ and $b^T=-2\cdot x^TW/||Wx||^2$ and $a=Wx$. 
The Jacobian $J$ from \Cref{lemma:jac} is then $J=M+ab^T$. 
The determinant of $J$ is then: 
\begin{align*}
    \det(J)&=\det(M)(1+b^TM^{-1}a)\\
    &=\det(H(Wx)\cdot A)\left(1-2\frac{x^TW( H(Wx)\cdot A)^{-1}Wx}{||Wx||^2}\right)\\
    &=-\det(A)\left(1+2\frac{x^TWA^{-1}Wx}{||Wx||^2}\right).
\end{align*}
This is true because $H(Wx)^{-1}=H(Wx)$, $H(Wx)\cdot Wx=-Wx$ and $\det(H(Wx))=-1$.
\end{proof}
%During the derivations we found it quite useful to write every step down in PyTorch and compare the found Jacobian against PyTorch autograd.  To help reviewers inspect our derivations we add the implementation of the Jacobian and a test case against autograd (\Cref{app:autograd}). 

We can now use \Cref{lemma:jacdet} to investigate when the Jacobian determinant is non-zero. In particular, the Jacobian determinant must be non-zero if both $\det(A)\neq 0$ and $1+2v^TA^{-1}u/||u||^2\neq 0$. 
In the following lemma, we prove that both are non-zero if $W=W^T$ and $3/2\cdot \lambda_{\min}> \lambda_{\max}$.
%This makes $A$ negative definite and thus invertible. Furthermore, $while guaranteing , To see this, note that $W=W^T$ implies $A=A^T$.

\begin{restatable}{lemma}{lemmaeig}\label{lemma:eig}
\label{lemma:nonzero}
Let $W=W^T$ and $3/2 \cdot \lambda_{\min} > \lambda_{\max}$ then $\lambda_i(A^{-1})< -1/2$ for $A$ from \Cref{lemma:jacdet}. 
These conditions imply that $\det(A)\neq 0$ and $1+2v^TA^{-1}u/||u||^2 \neq 0$ with $v^T,u$ from \Cref{lemma:jacdet}
%Use $v^T,A,u$ from \Cref{lemma:jacdet}. If $W=W^T$ then $A=A^T$. If $2\cdot \lambda_{\min}(W)>\lambda_{\max}(W)$ then $\lambda_i(A^{-1})<-1/2$.  Finally, $\det(A)\neq 0$ and $1+2v^TA^{-1}u/||u||^2\neq 0$. %Finally, if $W=W^T$ and $2\cdot \lambda_{\min}>\lambda_{\max}$ then $\lambda_i(A^{-1})<-1/2$. 
\end{restatable}
\begin{proof}
We first show that the inequality $3/2\cdot \lambda_{\min}(W)>\lambda_{\max}(W)$ implies $\lambda_i(A^{-1})<-1/2$. 
\begin{align*}
\lambda_i(A^{-1}) = \frac{1}{\lambda_i(A)}
= \frac{1}{1 -2\frac{x^TW^Tx}{||Wx||^2}\lambda_i(W)}
\end{align*}
If $\gamma_i:=\frac{x^TW^Tx}{||Wx||^2}\cdot \lambda_i(W)\in (1/2, 3/2)$ we get that $1/(1-2\gamma_i)\in (-\infty, -1/2)$ so $\lambda_i(A^{-1})< -1/2$. 
If we let $y:=Wx$ we get 
$\frac{x^TW^Tx}{||Wx||^2}=\frac{y^TW^{-1}y}{||y||^2}$. 
This is the \emph{Rayleigh quotient} of $W^{-1}$ at $y$, 
which for $W=W^T$ is within $[\lambda_{\min}(W^{-1}),\lambda_{\max}(W^{-1})]$. 
Therefore $\gamma_i \in [\frac{1}{\lambda_{\max}(W)}, \frac{1}{\lambda_{\min}(W)}]\cdot \lambda_i(W)$. 
Note first that $\gamma_{\min}\le 1$ and $\gamma_{\max} \ge 1$. It is left to show that $\gamma_{\min}\ge \lambda_{\min}/\lambda_{\max}>1/2$ and $\gamma_{\max}\le \lambda_{\max}/\lambda_{\min} < 3/2$. Both conditions on eigenvalues are met if $3/2\cdot \lambda_{\min} > \lambda_{\max}$.

We now want to show that $\det(A)\neq 0$ and $1+2v^TA^{-1}u/||u||^2\neq 0$. 
First, notice that $\det(A)=\prod_{i=1}^d\lambda_i(A)\neq 0$ since $\lambda_i(A)<-1/2$. Second, note that $W=W^T$ implies that the $v^T$ from \Cref{lemma:jacdet} can be written as $v^T=x^TW=x^TW^T=u^T$. 
This means we only need to ensure $u^TA^{-1}u/||u||^2$, the Rayleigh quotient of $A^{-1}$ at $u$, is different to $-1/2$. 
%To use the eigenvalue bound on the Rayleigh quotient we need $A=A^T$. 
But $W=W^T$ implies $A=A^T$ because $A=I-2x^TW^Tx/||Wx||^2\cdot W$.
The Rayleigh quotient is therefore bounded by $[\lambda_{\min}(A^{-1}), \lambda_{\max}(A^{-1})]$, which means it is less than $-1/2$ since $\lambda_{i}(A^{-1})<-1/2$. 
We can then conclude that also $1+2v^TA^{-1}u/||u||^2=1+2u^TA^{-1}u/||u||^2 < 1+2\cdot -1/2=0$. 
\end{proof}
So $\det(J)\neq 0$ by \Cref{lemma:eig} and \Cref{lemma:jacdet}, which by \Cref{thm:inv} implies invertibility (\Cref{thm:tinv}). 

\paragraph{Remark. }Note that the constraints $W=W^T$ and $3/2\cdot \lambda_{\min}>\lambda_{\max}$ were introduced only to guarantee $\det(A)\neq 0$ and $1+2v^TA^{-1}u/||u||^2\neq 0$. 
Any argument or constraints on $W$ that ensures $\det(A)\cdot (1+v^TA^{-1}u/||u||^2)\neq 0$ are thus sufficient to conclude $f(x)$ is invertible. 

\section{Experiments}\label{sec:exp}
We compare a single auxiliary reflections against $d$ normal reflections when training Fully Connected Neural Networks ($d=784$), Recurrent Neural Networks ($d=170)$ and Normalizing Flows ($d=48$). 
The experiments demonstrate that neural networks with a single auxiliary reflections attain similar performance to neural networks with many normal reflections.
All plots show means and standard deviations over 3 runs. 
See \Cref{app:details} for experimental details.

\subsection{Fully Connected Neural Networks}\label{exp:fcn}
We trained four different Fully Connected Neural Networks (FCNNs) for classification on MNIST. 
We compared a FCNN with $6$ auxiliary reflections against a FCNN with $6$ orthogonal matrices each represented by $784$ normal reflections. 
For completeness, we also trained two FCNNs where the $6$ orthogonal matrices where attained by the matrix exponential and Cayley map, respectively, as done in \citep{exp,cayley}. 
%This comparison is relevant because products of Householder matrices are orthogonal. 
The FCNN with auxiliary reflections attained slightly better validation error, see (\Cref{fig:fcn} left).  
%\alex{one sentence about why orthogonal matrices because reflectionscomposition of reflections are orthogonal. } 
Furthermore, we found the auxiliary reflections were $2$ to $6$ times faster than competing methods, see (\Cref{fig:fcn} right). This was true even though we used \citep{fasth} to speed up the sequential Householder reflections. See \Cref{app:fcn} for further details.

\subsection{Recurrent Neural Networks}\label{exp:rnn}
We trained three Recurrent Neural Networks (RNNs) for classification on MNIST as done in \citep{hhrnn}. 
The RNNs had a transition matrix represented by one auxiliary reflection, one normal reflection and $170$ auxiliary reflections.
See (\Cref{fig:rnn} left) for a validation error during training, including the model from \citep{hhrnn}. 
As indicated by the red curve, using only one normal reflection severely limits the transition matrix. 
In the right plot, we magnify the first 20 epochs to improve readability.
The RNNs with $1$ auxiliary reflection attains similar mean validation accuracy to the RNNs with $170$ normal reflections. 
%The RNN with $1$ normal reflection performed very poorly. 
%We believe this happens primarily because $1$ normal reflection is too restrictive. % , however, it is possible further optimization and/or better initialization would improve. 
See \Cref{app:rnn} for further details. 

%\clearpage

\begin{figure}[t!]
    \centering
    \includegraphics[width=1\textwidth]{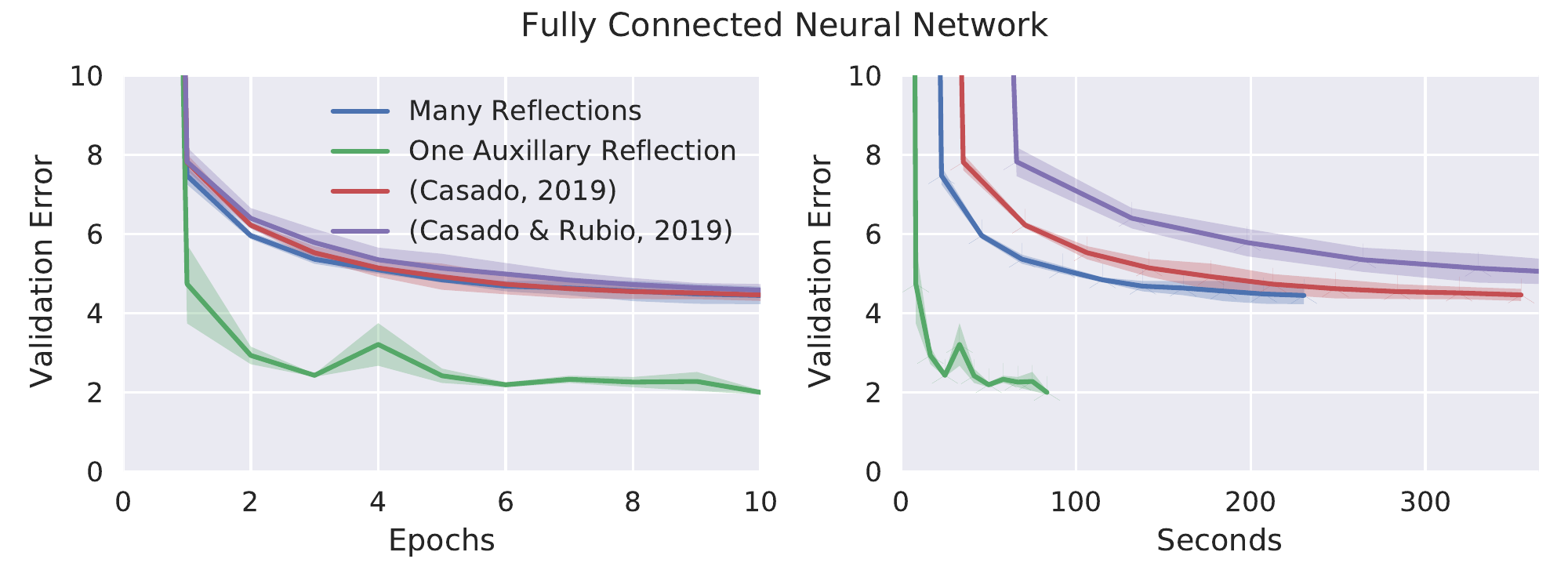}
    \caption{MNIST validation classification error in \% over epochs (left) and over time (right). Lower error mean better performance. 
    % Plot shows means and standard deviations over three runs of each method. 
    %\alex{reshape so it uses full width but stays one page}
    %\frederik{We need do describe our hardware setup somewhere}
    }
    \label{fig:fcn}
\end{figure}
\begin{figure}[t!]
    \centering
    \includegraphics[width=1\textwidth]{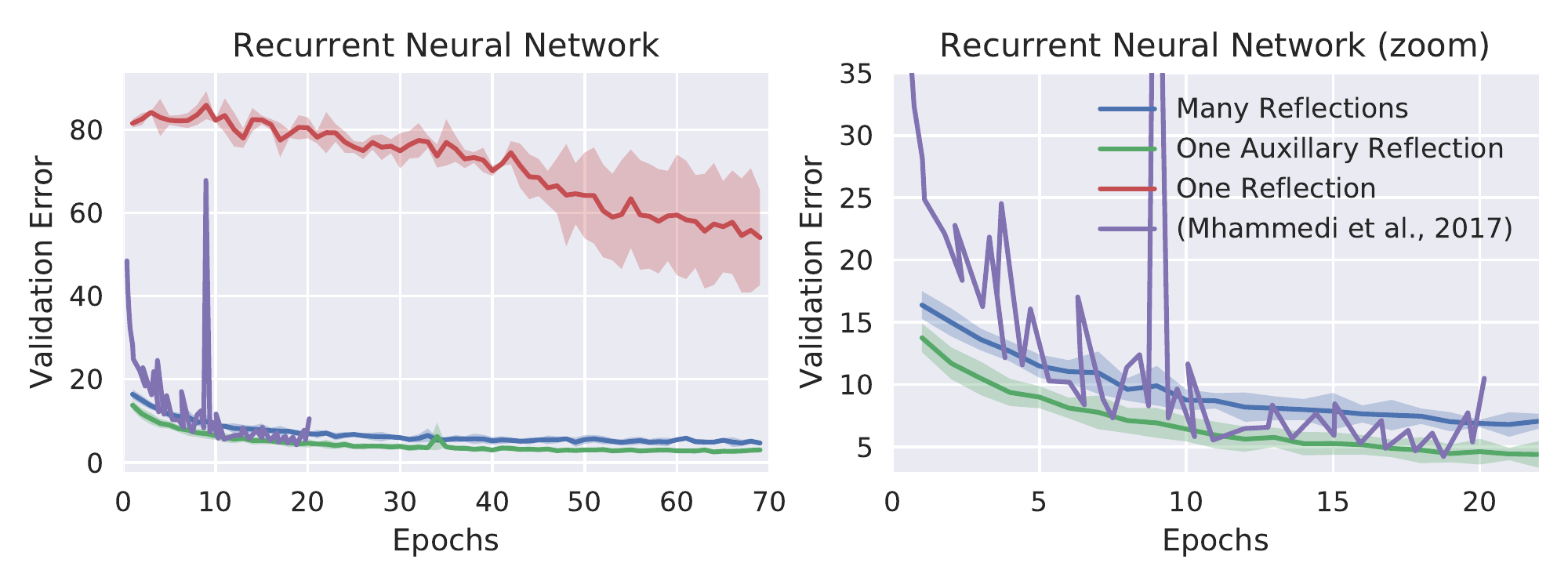}
    \caption{
        MNIST validation classification error for RNNs performing classification as done in \citep{hhrnn}. To improve readability, the right plot magnifies the first 20 epochs of the left. Lower error means better performance. 
     }
    \label{fig:rnn}
\end{figure}

\subsection{Normalizing Flows and Convolutions}\label{exp:nf}
We initially trained two Normalizing Flows (NFs) on CIFAR10. 
Inspired by \citep{emerging}, we used reflections to parameterize the 1x1 convolutions of an NF called Glow \citep{glow}, see \Cref{app:glow} for details. 
We trained an NF with many reflections and an NF with a single auxiliary reflection constrained to ensure invertible (see \Cref{subsec:jac}). 
%Validation NLL during training is shown in (\Cref{fig:nf} left). 
The single auxiliary reflection attained worse validation NLL compared to the model with $48$ normal reflections. 

We suspected the decrease in performance was caused by the restrictions put on the weight matrices $W_i$ of the auxiliary reflections to enforce invertibility, i.e., $W_i=W_i^T$ and $3/2\cdot \lambda_{\min}(W_i)>\lambda_{\max}(W_i)$. 
To investigate this suspicion, we trained a model with no constraints on $W$. 
This improved performance to the point were one auxiliary reflections tied with many normal reflections (see \Cref{fig:nf} left). 

Even though the resulting auxiliary reflections are not provably invertible, 
we found that Newtons method consistently computed the correct inverse. 
Based on this observation, we conjecture that the training dynamics caused the auxiliary reflections to remain invertible. By this we mean that the auxiliary reflections were initialized with non-zero Jacobian determinants (see \Cref{app:glow}) and the loss function (\Cref{equ:nll}) encourages the auxiliary reflections to increase their Jacobian determinants during training. %the network to increase the Jacobian determinants during training.  
%Note that the loss function contains the Jacobian determinant, so which means In particular, the loss function becomes infinite if the Jacobian determinant is zero on a training example. 
%In particular, the loss function is becomes infinite if the Jacobian determinant is zero on a training examplegives infinite loss if the encourages the Jacobian determinant to large on the training examples, be non-zero on the training examples, since zero Jacobian determinant implies infinite loss.
Since Newtons method consistently computed the correct inverse, we were able to generate samples from all models, see (\Cref{fig:nf} right). 

\begin{figure}[t!]
    \centering
    \includegraphics[width=\textwidth]{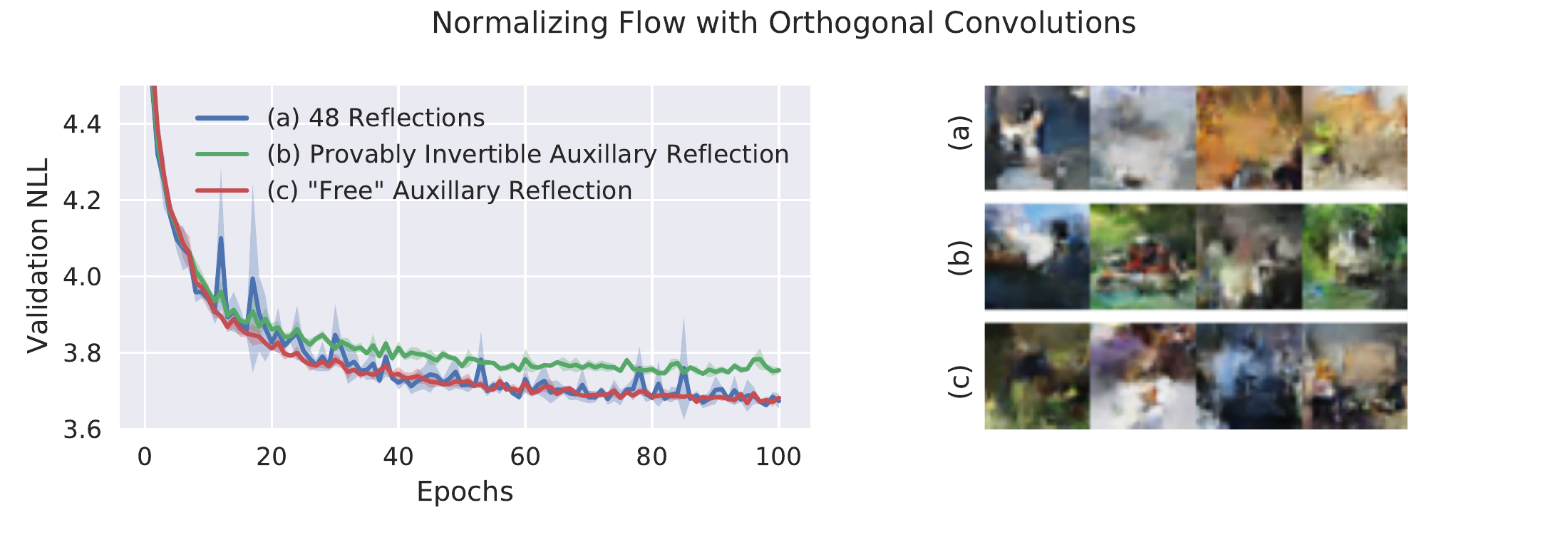}
    \caption{(Left) Validation negative log-likelihood (NLL) of three Normalizing Flows on CIFAR10. NLL is reported in bits per dimension, lower values mean better performance. (Right) Samples generated by different models, this required computing the inverse of auxiliary reflections. }
    \label{fig:nf}
\end{figure}

%with a 1x1 convolution in its QR decomposition. 
%In particular, their experiment suggest that 1x1 convolutions can be parameterized in a QR decomposition without degrading performance. 
%To this end, they represent the Q matrix as a product of up to $48$ reflections. 

%We train a Glow with $48$ reflections, one reflection and one auxiliary reflection. To isolate the differences due to changes in reflections, we remove the rectangular matrix $R$ used for the QR decomposition. We find this has a small impact on validation nll. We plot mean validation loss over three repetitions with standard deviaton as error bars in \Cref{fig:nf}. 
%
%The normalizing flow with one auxiliary reflection attained almost the same performance as $48$ normal reflections. We suspected the difference in performance was caused by the constraint that $2\lambda_{\min}<\lambda_{\max}$ used to guarantee invertibility. 
%To investigate this we ran an experiment without the constraint, and found that in this case a single auxiliary reflection attains the same performance as $48$ normal reflections, even though the function is not provably invertible. 

\section{Related Work}
\paragraph{Orthogonal Weight Matrices. }
Orthogonal weight matrices have seen widespread use in deep learning. For example, they have been used in Normalizing Flows \citep{emerging}, Variational Auto Encoders \citep{sylvesternf}, Recurrent Neural Networks \citep{hhrnn} and Convolutional Neural Networks \citep{orthcnn}. 

\paragraph{Different Approaches. }
There are several ways of constraining weight matrices to remain orthogonal. 
For example, previous work have used Householder reflections \citep{hhrnn}, the Cayley map \citep{cayley} and the matrix exponential \citep{exp}. 
These approaches are sometimes referred to as \emph{hard orthogonality constraints}, as opposed to \emph{soft orthogonality constraints}, which instead provide approximate orthogonality by using, e.g., regularizers like $||WW^T-I||_F$ (see \citep{orthcnn} for a comprehensive review). 

\paragraph{Reflection Based Approaches. }
The reflection based approaches introduce sequential computations, which is, perhaps, their main limitation. Authors often address this by reducing the number of reflections, as done in, e.g., \citep{hhnf,hhrnn,sylvesternf}. 
This is sometimes undesirable, as it limits the expressiveness of the orthogonal matrix. 
This motivated previous work to construct algorithms that increase parallelization of Householder products, see, e.g., \citep{fasth,cwy}. 

\paragraph{Similar Ideas. }
Normalizing Flows have been used for variational inference, see, e.g., \citep{hhnf,sylvesternf}. 
Their use of reflections is very similar to auxiliary reflections, however, there is a very subtle difference which has fundamental consequences. 
For a full appreciation of this difference, the reader might want to consult the schematic in \citep[Figure 1]{hhnf}, however, we hope that the text below clarifies the high-level difference. 

Recall that auxiliary reflections compute $H(Wx)x$ so $H(Wx)$ can depend on $x$. 
In contrast, the previous work on variational inference instead compute $H(v)z$ where $v$ and $z$ both depend on $x$. 
This limits $H(v)$ in that it can not explicitly depend on $z$. 
While this difference is subtle, it means our proof of \Cref{thm:orth} does not hold for reflections as used in \citep{hhnf}.

\section{Conclusion}
In theory, we proved that a single auxiliary reflection is as expressive as any number of normal reflections. In practice, we demonstrated that a single auxiliary reflection can attain similar performance to many normal reflections when training Fully Connected Neural Networks, Recurrent Neural Networks and Normalizing Flows. 
For Fully Connected Neural Networks, we reduced training time by a factor between $2$ and $6$ by using auxiliary reflections instead of previous approaches to orthogonal matrices \citep{hhrnn,cayley,exp}. % changing many normal reflections to a single auxiliary reflection. %  using a single auxiliary reflection using auxiliary reflections reduced training time by between $2\times $ and $6\times$ relative to competing methods. %can substantially reduce training time, e.g. $2$ to $6$ times. %We were not able to attain similar performance using a single normal reflection, which suggests the single reflection needs to be computed by an auxiliary network. 

\bibliography{main}

\begin{thebibliography}{24}
\providecommand{\natexlab}[1]{#1}
\providecommand{\url}[1]{\texttt{#1}}
\expandafter\ifx\csname urlstyle\endcsname\relax
  \providecommand{\doi}[1]{doi: #1}\else
  \providecommand{\doi}{doi: \begingroup \urlstyle{rm}\Url}\fi

\bibitem[Arjovsky et~al.(2016)Arjovsky, Shah, and Bengio]{urnn}
Martin Arjovsky, Amar Shah, and Yoshua Bengio.
\newblock {Unitary Evolution Recurrent Neural Networks}.
\newblock In \emph{ICML}, 2016.

\bibitem[Bansal et~al.(2018)Bansal, Chen, and Wang]{orthcnn}
Nitin Bansal, Xiaohan Chen, and Zhangyang Wang.
\newblock {Can We Gain More From Orthogonality Regularizations in Training Deep
  Networks?}
\newblock In \emph{NeurIPS}, 2018.

\bibitem[Berg et~al.(2018)Berg, Hasenclever, Tomczak, and Welling]{sylvesternf}
Rianne van~den Berg, Leonard Hasenclever, Jakub~M Tomczak, and Max Welling.
\newblock {Sylvester Normalizing Flows for Variational Inference}.
\newblock \emph{Conference on Uncertainty in Artificial Intelligence (UAI)},
  2018.

\bibitem[Casado(2019)]{exp}
Mario~Lezcano Casado.
\newblock {Trivializations for Gradient-Based Optimization on Manifolds}.
\newblock In \emph{NeurIPS}, 2019.

\bibitem[Dinh et~al.(2015)Dinh, Krueger, and Bengio]{nice}
Laurent Dinh, David Krueger, and Yoshua Bengio.
\newblock {NICE:} non-linear independent components estimation.
\newblock In \emph{ICLR, Workshop Proceedings}, 2015.

\bibitem[Dinh et~al.(2017)Dinh, Sohl-Dickstein, and Bengio]{realnvp}
Laurent Dinh, Jascha Sohl-Dickstein, and Samy Bengio.
\newblock {Density Estimation using Real NVP}.
\newblock In \emph{ICLR}, 2017.

\bibitem[Hinton et~al.(2012)Hinton, Srivastava, and Swersky]{rmsprop}
Geoffrey Hinton, Nitish Srivastava, and Kevin Swersky.
\newblock {(RMSProp) Neural Networks for Machine Learning Lecture 6a: Overview
  of Mini-Batch Gradient Descent}.
\newblock 2012.

\bibitem[Ho et~al.(2019)Ho, Chen, Srinivas, Duan, and Abbeel]{flowpp}
Jonathan Ho, Xi~Chen, Aravind Srinivas, Yan Duan, and Pieter Abbeel.
\newblock {Flow++: Improving flow-based generative models with variational
  dequantization and architecture design}.
\newblock In \emph{ICML}, 2019.

\bibitem[Hoogeboom et~al.(2019)Hoogeboom, Van Den~Berg, and Welling]{emerging}
Emiel Hoogeboom, Rianne Van Den~Berg, and Max Welling.
\newblock {Emerging Convolutions for Generative Normalizing Flows}.
\newblock In \emph{ICML}, 2019.

\bibitem[Kingma \& Ba(2015)Kingma and Ba]{adam}
Diederik~P Kingma and Jimmy Ba.
\newblock {Adam: A Method for Stochastic Optimization}.
\newblock In \emph{ICLR}, 2015.

\bibitem[Kingma \& Dhariwal(2018)Kingma and Dhariwal]{glow}
Durk~P Kingma and Prafulla Dhariwal.
\newblock {Glow: Generative Flow with Invertible 1x1 Convolutions}.
\newblock In \emph{NeurIPS}, 2018.

\bibitem[Krantz \& Parks(2012)Krantz and Parks]{hadamardthm}
Steven~G Krantz and Harold~R Parks.
\newblock \emph{{The Implicit Function Theorem: History, Theory, and
  Applications}}.
\newblock Springer Science \& Business Media, 2012.

\bibitem[Krizhevsky et~al.(2009)Krizhevsky, Hinton, et~al.]{cifar}
Alex Krizhevsky, Geoffrey Hinton, et~al.
\newblock {Learning Multiple Layers of Features from Tiny Images}.
\newblock 2009.

\bibitem[Lee(2013)]{sdsmooth}
John~M Lee.
\newblock {Smooth Manifolds}.
\newblock In \emph{Introduction to Smooth Manifolds}. 2013.

\bibitem[Lezcano-Casado \& Mart{\'\i}nez-Rubio(2019)Lezcano-Casado and
  Mart{\'\i}nez-Rubio]{cayley}
Mario Lezcano-Casado and David Mart{\'\i}nez-Rubio.
\newblock {Cheap Orthogonal Constraints in Neural Networks: A Simple
  Parametrization of the Orthogonal and Unitary Group}.
\newblock \emph{ICML}, 2019.

\bibitem[Likhosherstov et~al.(2020)Likhosherstov, Davis, Choromanski, and
  Weller]{cwy}
Valerii Likhosherstov, Jared Davis, Krzysztof Choromanski, and Adrian Weller.
\newblock {CWY Parametrization for Scalable Learning of Orthogonal and Stiefel
  Matrices}.
\newblock \emph{arXiv preprint arXiv:2004.08675}, 2020.

\bibitem[Mathiasen et~al.(2020)Mathiasen, Hvilsh{\o}j, J{\o}rgensen, Nasery,
  and Mottin]{fasth}
Alexander Mathiasen, Frederik Hvilsh{\o}j, Jakob~R{\o}dsgaard J{\o}rgensen,
  Anshul Nasery, and Davide Mottin.
\newblock {F}aster {O}rthogonal {P}arameterization with {H}ouseholder
  {M}atrices.
\newblock In \emph{ICML, Workshop Proceedings}, 2020.

\bibitem[Mhammedi et~al.(2017)Mhammedi, Hellicar, Rahman, and Bailey]{hhrnn}
Zakaria Mhammedi, Andrew Hellicar, Ashfaqur Rahman, and James Bailey.
\newblock {Efficient Orthogonal Parametrisation of Recurrent Neural Networks
  Using Householder Reflections}.
\newblock In \emph{ICML}, 2017.

\bibitem[Miyato et~al.(2018)Miyato, Kataoka, Koyama, and Yoshida]{snorm}
Takeru Miyato, Toshiki Kataoka, Masanori Koyama, and Yuichi Yoshida.
\newblock {Spectral Normalization for Generative Adversarial Networks}.
\newblock In \emph{ICLR}, 2018.

\bibitem[Paszke et~al.(2017)Paszke, Gross, Chintala, Chanan, Yang, DeVito, Lin,
  Desmaison, Antiga, and Lerer]{pytorch}
Adam Paszke, Sam Gross, Soumith Chintala, Gregory Chanan, Edward Yang, Zachary
  DeVito, Zeming Lin, Alban Desmaison, Luca Antiga, and Adam Lerer.
\newblock {Automatic Differentiation in PyTorch}.
\newblock 2017.

\bibitem[Petersen \& Pedersen(2012)Petersen and Pedersen]{cookbook}
Kaare~Brandt Petersen and Michael~Syskind Pedersen.
\newblock {The Matrix Cookbook}, 2012.
\newblock Technical University of Denmark, Version 20121115.

\bibitem[Tomczak \& Welling(2016)Tomczak and Welling]{hhnf}
Jakub~M Tomczak and Max Welling.
\newblock {Improving Variational Auto-Encoders using Householder Flow}.
\newblock \emph{arXiv preprint arXiv:1611.09630}, 2016.

\bibitem[Wang(2015)]{hhqr}
Ruye Wang.
\newblock {Lecture Notes: Householder Transformation and QR Decomposition},
  2015.
\newblock URL \url{http://fourier.eng.hmc.edu/e176/lectures/NM/node10.html}.

\bibitem[Zhang et~al.(2018)Zhang, Lei, and Dhillon]{svdnn}
Jiong Zhang, Qi~Lei, and Inderjit Dhillon.
\newblock {Stabilizing Gradients for Deep Neural Networks via Efficient SVD
  Parameterization}.
\newblock In \emph{ICML}, 2018.

\end{thebibliography}
\bibliographystyle{iclr2020_conference}

\clearpage 
\appendix
\section{Appendix}

\subsection{Proofs}

\subsubsection{\Cref{thm:orth}}
\label{app:lemma:ref}
Our proof of \Cref{thm:orth} is an follows \Cref{lemma:refl} which we state below. % since $||Ux||=||x||$. 

\thmorth*
\begin{proof}
Let $W=I-U$ then $H(Wx)x=H(x-Ux)x=Ux$ for all $x\in \mathbb{R}^d$ since $||Ux||=||x||$. 
\end{proof}

\begin{restatable}{lemma}{lemmarefl}\label{lemma:refl}
Let $||x||=||y||$ then $H(x-y)x=y$. 
\end{restatable}
%\begin{proof} See below.  \end{proof}
\begin{proof}
The result is elementary, see, e.g., \citep{hhqr}. For completeness, we derive it below. 
\begin{align*}
    H(x-y)x&=x-2\frac{(x-y)(x-y)^T}{||x-y||^2}x\\
    &=x-2\frac{xx^T + yy^T - xy^T - yx^T}{x^Tx + y^Ty - 2x^Ty}x\\
    &=x-2\frac{xx^Tx + yy^Tx - xy^Tx - yx^Tx}{x^Tx + y^Ty - 2x^Ty}\\
    &=x-2\frac{x||x||+ yy^Tx - xy^Tx - y||x||}{2||x||^2- 2x^Ty}\\
    &=x-\frac{(x-y)||x||^2+ (y-x)(y^Tx) }{||x||^2- x^Ty}\\
    &=x-\frac{(x-y)||x||^2+ (y-x)(y^Tx) }{||x||^2- x^Ty}\\
    &=x-\frac{(x-y)(||x||^2- x^Ty) }{||x||^2- x^Ty}\\
    &=x-(x-y)= y 
\end{align*}

\end{proof}

\clearpage

\subsection{PyTorch Examples and Test Cases}
%We found reviewers rarely have time to install dependencies and run code. 
To ease the workload on reviewers, we opted to use small code snippets that can be copied into \href{www.colab.research.google.com}{www.colab.research.google.com} and run in a few seconds without installing any dependencies. 
Some PDF viewers do not copy line breaks, we found viewing the PDF in Google Chrome works. 

%\subsubsection{Example: Normal and Auxiliary Reflection}\label{app:refl}
%
%\begin{python}
%import torch
%torch.manual_seed(42)
%d = 4
%
%# Create random test-case. 
%W = torch.nn.Parameter(torch.empty((d, d)).uniform_(-1, 1))
%x = torch.zeros((d, 1)).uniform_(-1, 1)
%
%# Compute transformation using 'd' sequential Householder reflections. 
%def H(v): return I - 2 * v @ v.T / (v.T @ v)
%y = x 
%for i in range(d):  y = H(W[:, i:i+1]) @ y 
%
%# Construct the same transform using a single auxiliary reflection 'f'. 
%def f(x): return H(W @ x) @ x
%\end{python}

\subsubsection{Test Case: Inverse using Newtons Method} \label{app:inv}
Given $y$ we compute $x$ such that $H(Wx)x=y$ using Newtons method. 
To be concrete, the code below contains a toy example where $x\in \mathbb{R}^4$ and $W=I+VV^T/(2\cdot \sigma_{\max}(VV^T)) \in \mathbb{R}^{4\times 4}$. 
The particular choice of $W$ makes $H(Wx)x$ invertible, because $\lambda_i(W)=1+\lambda_i(VV^T)=1+\sigma_i(VV^T)\in [1, 3/2)$ because $VV^T$ is positive definite. 
Any possible way of choosing the eigenvalues in the range $[1, 3/2)$ guarantees that $3/2\cdot \lambda_{\min} > \lambda_{\max}$ which implies invertibility by \Cref{thm:tinv}. %as shown in \Cref{sec:proof} implies invertibility. %because $VV^T$ is positive definite for any $V\in\mathbb{R}^{d \times d}$. For further details see \Cref{sec:proof}.
\begin{python}
import torch
print("torch version: ", torch.__version__)
torch.manual_seed(42)
d = 4
# Create random test-case.
I = torch.eye(d)
V = torch.zeros((d, d)).uniform_()
x = torch.zeros((d, 1)).uniform_()
W = I + V @ V.T / torch.svd(V @ V.T)[1].max()
# Define the function f(x)=H(Wx)x.
def H(v): return torch.eye(d) - 2 * v @ v.T / (v.T @ v)
def f(x): return H(W @ x ) @ x
# Print input and output
print("x\t\t", x.data.view(-1).numpy())
print("f(x)\t", f(x).data.view(-1).numpy())
print("")

# Use Newtons Method to compute inverse.
y  = f(x)
xi = y
for i in range(10):
  print("[%.2i/%.2i]"%(i+1, 10), xi.data.view(-1).numpy())
  # Compute Jacobian using Theorem 3.
  A = torch.eye(d) - 2* (xi.T @ W.T @ xi) / torch.norm(W @ xi)**2 * W
  J = -2*W @ xi @ xi.T @ W/torch.norm(W@xi)**2 + H(W @ xi) @ A
  xi = xi - torch.inverse(J) @ (f(xi)- y)
assert torch.allclose(xi, x, atol=10**(-7))
print("The two vectors are torch.allclose")
\end{python}

\begin{python}
torch version:  1.6.0+cu101
x		 [0.8854429  0.57390445 0.26658005 0.62744915]
f(x)	 [-0.77197534 -0.49936318 -0.5985155  -0.6120473 ]

[01/10] [-0.77197534 -0.49936318 -0.5985155  -0.6120473 ]
[02/10] [ 0.72816867  0.78074205 -0.02241153  1.0435152 ]
[03/10] [0.7348436  0.6478982  0.14960966 0.8003925 ]
[04/10] [0.8262452 0.6155189 0.2279686 0.6997254]
[05/10] [0.8765415 0.5831212 0.2592551 0.640691 ]
[06/10] [0.8852093  0.5742159  0.26631045 0.6278922 ]
[07/10] [0.88543946 0.5739097  0.26658094 0.62744874]
[08/10] [0.88544315 0.57390547 0.2665805  0.6274475 ]
[09/10] [0.885443   0.57390594 0.26658088 0.6274466 ]
[10/10] [0.8854408  0.57390743 0.2665809  0.6274484 ]
The two vectors are torch.allclose
\end{python}
\clearpage 
\paragraph{\Cref{fig:inv}. } 
\Cref{fig:inv} contains reconstructions $n^{-1}(n(x))$ of the variant of Glow \citep{glow} used in \Cref{exp:nf}. 
The Glow variant has 1x1 convolutions with auxiliary reflections, i.e., 
for an input $x\in\mathbb{R}^{c \times h \times w}$ where $(c,h,w)$ are (channels, heigh, width) it computes $z_{:,i,j}=H(Wx_{:,i,j})x_{:,i,j}\in\mathbb{R}^c$ where $i=1,...,h$ and $j=1,...,w$. 
Computing the inverse required computing the inverse of the auxiliary 1x1 convolutions, i.e., compute $x_{:,i,j}$ given $W$ and $z_{:, i, j}\; \forall i,j$. 
The weights were initialized as done in the above toy example.

\subsubsection{Test Case: Jacobian and Autograd}
\label{app:autograd}

\begin{python}
import torch
print("torch version: ", torch.__version__)
torch.manual_seed(42)

# Create random test-case. 
d = 4
W = torch.zeros((d, d)).uniform_(-1, 1) 
x = torch.zeros((d, 1)).uniform_(-1, 1)
I = torch.eye(d)

# Compute Jacobian using autograd. 
def H(v): return I - 2 * v @ v.T / (v.T @ v)
def f(x): return H(W @ x ) @ x
J     = torch.autograd.functional.jacobian(f, x)[:, 0, :, 0]
print(J)

# Compute Jacobian using Lemma 4. 
A  = I - 2* (x.T @ W.T @ x) / torch.norm(W @ x)**2 * W 
J_ = H(W @ x) @ A -2*W @ x @ x.T @ W/torch.norm(W@x)**2 
print(J_)
    
# Test the two matrices are close. 
assert torch.allclose(J, J_, atol=10**(-5))
print("The two matrices are torch.allclose")
\end{python}

\begin{python}
torch version:  1.6.0+cu101
tensor([[ 0.2011, -1.4628,  0.7696, -0.5376],
        [ 0.3125,  0.6518,  0.7197, -0.5997],
        [-1.0764,  0.8388,  0.0020, -0.1107],
        [-0.8789, -0.3006, -0.4591,  1.3701]])
tensor([[ 0.2011, -1.4628,  0.7696, -0.5376],
        [ 0.3125,  0.6518,  0.7197, -0.5997],
        [-1.0764,  0.8388,  0.0020, -0.1107],
        [-0.8789, -0.3006, -0.4591,  1.3701]])
The two matrices are torch.allclose
\end{python}

\clearpage

\section{Experimental details}\label{app:details}
In this section, we specify the details of the three experiments presented in the \Cref{sec:exp}.
The experiments were run on a single NVIDIA RTX 2080 Ti GPU and Intel Xeon Silver 4214 CPU @ 2.20GHz. 

\subsection{Fully Connected Neural Networks}\label{app:fcn}
For the experiment in \Cref{exp:fcn} we trained four Fully Connected Neural Networks (FCNNs) as MNIST classifiers. 
All FCNNs had the same structure which we now explain.
Inspired by \citep{svdnn} the layers of the FCNNs were parametrized in their Singular Value Decomposition (SVD). 
This just means each layer consisted of two orthogonal matrices $U,V\in\mathbb{R}^{784\times 784}$ and a diagonal matrix $\Sigma\in\mathbb{R}^{784\times 784}$, so the forward pass computes $y=U\Sigma V^Tx$. 
The FCNNs had three such fully connected layers with relu non-linearity in between, and a final linear layer of shape $W\in\mathbb{R}^{784\times 10}$. 
We used the Adam optimizer \citep{adam} with default parameters\footnote{Default parameters of the Adam implementation in PyTorch 1.6 \citep{pytorch}. }  to minimize cross entropy. 
To speed up the network with $784$ normal reflections, we used the FastH algorithm from \citep{fasth}. 
%This choice does not change validation loss compared to sequentially evaluating the reflections, because FastH computes the same thing, just faster \citep{fasth}. 
For the network with auxiliary reflections, we had $U,V$ be auxiliary reflections instead of orthogonal matrices. 
In all experiments, we initialized the singular values $\Sigma_{ij}\sim U(0.99, 1.01)$. 

We used orthogonal matrices with reflections, the Cayley transform and the matrix exponential as done in \citep{hhrnn,exp,cayley}, respectively. 
The orthogonal matrices are constructed using a weight matrix $W$. 
In all cases, we initialized $W_{ij}\sim U(-\frac{1}{\sqrt{d}}, \frac{1}{\sqrt{d}})$. 
It is possible one could initialize $W_{ij}$ in a smarter way, which could change the validation error reported in \Cref{fig:fcn}. 
That said, we did try to initialize $W$ using the \emph{Cayley initialization} suggested by \citep{exp}. However, we did not find it improved performance. 

\subsection{Recurrent Neural Networks}\label{app:rnn}
For the experiment in \Cref{exp:rnn}, we trained three Recurrent Neural Networks as MNIST classifiers as done in \citep{urnn,hhrnn,svdnn,exp}. 
We used the open-source implementation from \citep{exp}.\footnote{\href{https://github.com/Lezcano/expRNN/}{https://github.com/Lezcano/expRNN/}}
They use a clever type of ``Cayley initialization'' to initialize the transition matrix $U$. 
We found it worked very well, so we choose to initialize both the normal and auxiliary reflections so they initially represented the same transition matrix $U$. 
For normal reflections, this can be done by computing $v_1,...,v_d$ so $H_1\cdots H_d=U$ by using the QR decomposition. 
For the auxiliary reflection, this can be done using $W=I-U$ so $H(Wx)x=Ux$ (see \Cref{thm:inv}). 

In \citep{exp}, they use $h_0=0$ as initial state and report ``\emph{We choose as the initial vector $h_0=0$ for simplicity, as we did not observe any empirical improvement when using the initialization given in \citep{urnn}.}'' 
We sometimes encountered division by zero with auxiliary reflections when $h_0=0$, so we used the initialization suggested by \citep{urnn} in all experiments. 

The open-source implementation \citep{exp} use RMSProp \citep{rmsprop} with different learning rates for the transition matrix and the remaining weights. 
This was implemented in PyTorch by using two RMSProp optimizers. 
We found training auxiliary reflectons to be more stable with Adam \citep{adam}. 
We believe this happens because the ``averaged gradients'' $v$ become very small due to the normalization term $||Wx||^2$ in $H(Wx)x=x-2Wxx^TW^Tx/||Wx||^2$. 
When $v$ becomes small the scaling $1/(\sqrt{v}+\epsilon)$ of RMSProp becomes very large. 
We suspect the $1/(\sqrt{v/(1-\beta_2^T)}+\epsilon)$ scaling used by Adam fixed the issue, which caused more stable training with Adam. 
This caused us to use Adam optimizer for the transition matrix instead of RMSProp for all the RNNs we trained. 

\clearpage 
\subsection{Normalizing Flow}
\label{app:glow}
For the experiment in \Cref{exp:nf}, we trained three Normalizing Flows as generative models on CIFAR10 as done in \citep{nice,realnvp,glow,flowpp}. 
We used an open-source PyTorch implementation of Glow \citep{glow}\footnote{\href{https://github.com/chrischute/glow}{https://github.com/chrischute/glow}} with default parameters, 
except for the number of channels ``-C'' and the number of steps ``-K.'' 
In particular, to decrease training time, we reduced ``-C'' from $512$ to $64$ and ``-K'' from $32$ to $8$. 
This caused an increase in validation NLL (worse performance) from $3.49$ to $3.66$ after $80$ epochs.

\paragraph{Auxiliary Reflections for 1x1 Convolutions. }
\citep{glow} suggested using invertible $1\times 1$ convolutions for Normalizing Flows. 
That is, given an input $x\in \mathbb{R}^{c\times h\times w }$ and kernel $W\in \mathbb{R}^{c \times c}$ they compute $z_{:,i,j}=Wx_{:,i,j}$ for all $i,j$. 
The resulting function is invertible if $W$ is, and it has Jacobian determinant $hw\det(W)$. 
It was suggested by \citep{emerging} to represent $W$ in its QR decomposition so $\det(W)=\det(QR)=\det(Q)\det(R)=\det(R)=\prod_iR_{ii}$. 
To this end, they represent the orthogonal matrix $Q$ as a product of reflections, in particular, they use $c=12,24,48$ reflections at different places in the network. 
The main goal of this experiment, was to compare $c=12,24,48$ normal reflections against a single auxiliary reflection, which computes $z_{:,i,j}=H(Wx_{:,i,j})x_{:,i,j}$ instead of $z_{:,i,j}=Wx_{:,i,j}$.
To isolate the difference in performance due to reflections, we further removed the rectangular matrix. %The goal of this experiment was to compare the performance between many normal reflection and a single auxiliary reflection, so we choose to remove the rectangular matrix to amplify the difference between normal reflections and auxiliary reflections. 

\paragraph{Provably Invertible. }One of the Normalizing Flows with auxiliary reflections had the weights of its auxiliary reflections constrained to ensure invertibility. 
In particular, we let each weight matrix be $W=I+VV^T$ and used spectral normalization $VV^T/(2\sigma_{\max}(VV^T))$ to ensure $\sigma_i(VV^T)< 1/2$. 
The largest singular value can be computed efficiently using power iteration \citep{snorm}. 
For ease of implementation, we circumvented using power iteration due to a known open issue in the official PyTorch implementation. 
We instead used \textsc{torch.symeig} to compute the largest singular value by computing the largest eigenvalue $\lambda_{\max}(VV^T)=\sigma_{\max}(VV^T)$, which holds because $VV^T$ is positive definite for $V\in \mathbb{R}^{c\times c}$. 
This was only possible because the matrices where at most $48\times 48$, for larger problems one would be forced to use the power iteration.

\paragraph{Initialization. }The open-source implementation of Glow initializes $W=Q$ with $Q$ from \textsc{torch.qr(torch.randn((c,c)))[0]}. 
For the case with normal reflections, we computed $v_1,...,v_c$ such that $H(v_1)\cdots H(v_c)=Q$ \citep{hhqr}. 
For the auxiliary reflection without constraints we let $W=I-Q$ such that $H(Wx)x=H(x-Qx)=Qx$ by \Cref{lemma:refl}. 

However, for the experiment with constraints on $W$, we could not initiallize $W=I-Q$ and instead used $W=I+VV^T$ where (initially) $V_{ij}\sim U(-\frac{1}{\sqrt{c}}, \frac{1}{\sqrt{c}})$. 
This increased error at initialization from $6.5$ to $8$. % from an initial performance of at initialization from In this case we found the initial performance had nll $8$ instead of $6.5$ by default initialization.  
We suspect this happens because the previous initialization of $W$ has complex eigenvalues which $W=I+VV^T$ does not (because it is symmetric). 
In practice, we mitigate the poor initialization by using an additional fixed matrix $Q=Q^T$ which does not change during training.
This is essentially the same as using a fixed permutation as done in \citep{realnvp}, but, instead of using a fixed permutation, we use a fixed orthogonal matrix. 
While using the fixed $Q$, we found simply initializing $V=I$ worked sufficiently well. %osufficiently well. % instead of $V_{ij}\sim U(-\frac{1}{\sqrt{c}}, \frac{1}{\sqrt{c}})$ worked well.

\end{document}